\documentclass{article}

\usepackage{microtype}
\usepackage{bm}
\usepackage{graphicx}
\usepackage{subcaption}
\usepackage{booktabs} 
\usepackage{hyperref}       
\usepackage{url}            
\usepackage{amsfonts}       
\usepackage{amsthm}  
\usepackage{amsmath}

\usepackage{amssymb}
\usepackage{algorithm}
\usepackage{algorithmic}
\newtheorem{assumption}{Assumption}
\newtheorem{lemma}{Lemma}
\newtheorem{theorem}{Theorem}
\newtheorem*{lemma*}{Lemma}
\newtheorem*{theorem*}{Theorem}

\usepackage{hyperref}

\usepackage[accepted]{icml2019}

\icmltitlerunning{On the Generalization Gap in Reparameterizable RL}

\begin{document}

\twocolumn[
\icmltitle{On the Generalization Gap in Reparameterizable Reinforcement Learning}




\begin{icmlauthorlist}
\icmlauthor{Huan Wang}{to}
\icmlauthor{Stephan Zheng}{to}
\icmlauthor{Caiming Xiong}{to}
\icmlauthor{Richard Socher}{to}
\end{icmlauthorlist}

\icmlaffiliation{to}{Salesforce Research, Palo Alto CA, USA}

\icmlcorrespondingauthor{Huan Wang}{huan.wang@salesforce.com}

\icmlkeywords{Reinforcement Learning, Generalization, Reparameterization, Overfit}

\vskip 0.3in
]



\printAffiliationsAndNotice{}  

\begin{abstract}
Understanding generalization in reinforcement learning (RL) is a significant challenge, as many common assumptions of traditional supervised learning theory do not apply. 
%
We focus on the special class of \textbf{reparameterizable RL} problems, where the trajectory distribution can be decomposed using the reparametrization trick.
%
%
For this problem class, estimating the expected return is efficient and the trajectory can be computed deterministically given peripheral random variables, 
which enables us to study reparametrizable RL using supervised learning and transfer learning theory.  
%
%
%
Through these relationships, we derive guarantees on the gap between the expected and empirical return for both intrinsic and external errors, based on Rademacher complexity as well as the PAC-Bayes bound. 
Our bound suggests the generalization capability of reparameterizable RL is related to multiple factors including ``smoothness'' of the environment transition, reward and agent policy function class.
We also empirically verify the relationship between the generalization gap and these factors through simulations.
\end{abstract}

\section{Introduction}
Reinforcement learning (RL) has proven successful in a series of applications such as games \citep{Silver16,Silver17,Mnih15,Vinyals17,OpenAI_dota}, robotics \citep{Kober2013}, recommendation systems \citep{Li10,Shani05}, resource management \citep{Mao2016,Mirhoseini18}, neural architecture design \citep{Baker17}, and more. However some key questions in reinforcement learning remain unsolved. One that draws more and more attention is the  issue of overfitting in reinforcement learning \citep{Sutton95,Cobbe18,Zhang18,Packer18,Zhang18b}. A model that performs well in the training environment, may or may not perform well when used in the testing environment. There is also a growing interest in understanding the conditions for model generalization and developing algorithms that improve generalization. 

In general we would like to measure how accurate an algorithm is able to predict on previously unseen data. One metric of interest is the gap between the training and testing loss or reward. 
It has been observed that such gaps are related to multiple factors: initial state distribution, environment transition, the level of ``difficulty'' in the environment, model architectures, and optimization. 
\citet{Zhang18} split randomly sampled initial states into training and testing and evaluated the performance gap in deep reinforcement learning. 
They empirically observed overfitting caused by the randomness of the environment, even if the initial distribution and the transition in the testing environment are kept the same as training. 
On the other hand, \citet{Farebrother18,Justesen18,Cobbe18} allowed the test environment to vary from training, and observed huge differences in testing performance. 
\citet{Packer18} also reported very different testing behaviors across models and algorithms, even for the same RL problem.

Although overfitting has been empirically observed in RL from time to time, theoretical guarantees on generalization, especially finite-sample guarantees, are still missing. 
In this work, we focus on \emph{on-policy RL}, where agent policies are trained based on episodes of experience that are sampled ``on-the-fly'' using the current policy in training. 
We identify two major obstacles in the analysis of on-policy RL.
First, the episode distribution keeps changing as the policy gets updated during optimization. 
Therefore, episodes have to be continuously redrawn from the new distribution induced by the updated policy during optimization.
For finite-sample analysis, this leads to a process with complex dependencies. 
Second, state-of-the-art research on RL tends to mix the errors caused by randomness in the environment and shifts in the environment distribution. 
We argue that actually these two types of errors are very different. One, which we call intrinsic error, is analogous to overfitting in supervised learning, and the other, called external error, looks more like the errors in transfer learning.

Our key observation is there exists a special class of RL, called \textbf{reparameterizable RL}, where randomness in the environment can be decoupled from the transition and initialization procedures via the reparameterization trick \citep{Kingma14}. 
Through reparameterization, an episode's dependency on the policy is ``lifted'' to the states. Hence, as the policy gets updated, episodes are deterministic given peripheral random variables. 
%
As a consequence, the expected reward in reparameterizable RL is connected to the Rademacher complexity as well as the PAC-Bayes bound. 
The reparameterization trick also makes the analysis for the second type of errors, i.e., when the environment distribution is shifted, much easier since the environment parameters are also ``lifted'' to the representation of states.

\paragraph{Related Work}
Generalization in reinforcement learning has been investigated a lot both theoretically and empirically. 
Theoretical work includes bandit analysis \citep{Agarwal14,Auer02,Auer09,beygelzimer11}, Probably Approximately Correct (PAC) analysis \citep{Jiang17,Dann17,Strehl09,Lattimore14} as well as minimax analysis \citep{Azar17,Chakravorty03}.
Most works focus on the analysis of regret and consider the gap between the expected value and optimal return. 
On the empirical side, besides the previously mentioned work, \citet{Whiteson11} proposes generalized methodologies that are based on multiple environments sampled from a distribution. 
\citet{Nair15} also use random starts to test generalization. 

Other research has also examined generalization from a transfer learning perspective. 
\citet{Lazaric12,Taylor09,Zhan15,Laroche17} examine model generalization across different learning tasks, and provide guarantees on asymptotic performance.

There are also works in robotics for transferring policy from simulator to real world and optimizing an internal model from data \citep{Kearns2002}, or works trying to solve abstracted or compressed MDPs \citep{Majeed18}.


\textbf{Our Contributions}:
\begin{itemize}
    \item A connection between (on-policy) reinforcement learning and supervised learning through the reparameterization trick. It simplifies the finite-sample analysis for RL, and yields Rademacher and PAC-Bayes bounds on Markov Decision Processes (MDP).
    \item Identifying a class of \textbf{reparameterizable RL} and providing a simple bound for ``smooth'' environments and models with a limited number of parameters.
    \item A guarantee for reparameterized RL when the environment is changed during testing. In particular we discuss two cases in environment shift: change in the initial distribution for the states, or the transition function.
\end{itemize}





\section{Notation and Formulation}\label{sec:notation}

We denote a Markov Decision Process (MDP) as a $5$-tuple $(\mathcal{S}, \mathcal{A}, \mathcal{P}, r, \mathcal{P}_0)$. Here $\mathcal{S}$ is the state space, $\mathcal{A}$ is the action-space, $\mathcal{P}(s, a, s'):\mathcal{S}\times \mathcal{A}\times \mathcal{S}\rightarrow [0,1]$ is the transition probability from state $s$ to $s'$ when taking action $a$, $r(s): \mathcal{S}\rightarrow \mathbb{R}$ represents the reward function, and $\mathcal{P}_0(s):\mathcal{S}\rightarrow [0,1]$ is the initial state distribution. 
Let $\pi(s)\in \Pi: \mathcal{S}\rightarrow \mathcal{A}$ be the policy map that returns the action $a$ at state $s$. 

We consider episodic MDPs with a finite horizon.
Given the policy map $\pi$ and the transition probability $\mathcal{P}$, the state-to-state transition probability is 
$\mathcal{T}_{\pi}(s, s') = \mathcal{P}(s, \pi(s), s')$. 
Without loss of generality, the length of the episode is $T+1$. 
We denote a sequence of states $[s_0, s_1, \dots, s_T]$ as $\boldsymbol s$. 
The total reward in an episode is $R(\boldsymbol s)=\sum_{t=0}^T \gamma^t r_t$, where $\gamma\in(0,1]$ is a discount factor and $r_t = r(s_t)$.

Denote the joint distribution of the sequence of states in an episode $\boldsymbol s=[s_0, s_1,\dots, s_T]$ as $\mathcal{D}_\pi$. 
Note $\mathcal{D}_\pi$ is also related to $\mathcal{P}$ and $\mathcal{P}_0$. 
In this work we assume $\mathcal{P}$ and $\mathcal{P}_0$ are fixed, so $\mathcal{D}_\pi$ is a function of $\pi$. Our goal is to find a policy that maximizes the expected total discounted reward (return):
\begin{align}
    \pi^\ast = \arg\max_{\pi\in \Pi}\mathbb{E}_{\boldsymbol s\sim \mathcal{D}_\pi} R(\boldsymbol s)=\arg\max_{\pi\in \Pi}\mathbb{E}_{\boldsymbol s\sim \mathcal{D}_\pi} \sum_{t=0}^T\gamma^t r_t. \label{obj:exp}
\end{align}

Suppose during training we have a budget of $n$ episodes, then the empirical return is
\begin{align}
     \hat{\pi} = \arg\max_{\pi\in \Pi, {\boldsymbol s}^i\sim \mathcal{D}_\pi}\frac{1}{n}\sum_{i=1}^n R({\boldsymbol s}^i)\label{obj:rl},
\end{align}
where ${\boldsymbol s}^i=[s_0^i,s_1^i,\dots, s_T^i]$ is the $i$th episode of length $T+1$. 
We are interested in the generalization gap 
\begin{align}
    \Phi=\left|\frac{1}{n}\sum_{i=1}^n R({\boldsymbol s}^i)- \mathbb{E}_{{\boldsymbol s}\sim\mathcal{D}'_{\hat{\pi}}} R(\boldsymbol s) \right|.
    \label{obj:gap}
\end{align}
Note that in (\ref{obj:gap}) the distribution $\mathcal{D}'_{\hat{\pi}}$ may be different from $\mathcal{D}_{\hat{\pi}}$ since in the testing environment $\mathcal{P}'$ as well as $\mathcal{P}'_0$ may be shifted compared to the training environment.

\section{Generalization in Reinforcement Learning v.s. Supervised Learning}
Generalization has been well studied in the supervised learning scenario. A popular assumption is that samples are independent and identically distributed $(x_i,y_i)\sim \mathcal{D}, \forall i\in\{1,2,\dots, n\}$. 
Similar to empirical return maximization discussed in Section \ref{sec:notation}, in supervised learning a popular algorithm is empirical risk minimization:
\begin{align}
    \hat{f} = \arg\min_{f\in\mathcal{F}}\frac{1}{n}\sum_{i=1}^n \ell(f, x_i, y_i),
\end{align}
where $f\in \mathcal{F}: \mathcal{X}\rightarrow \mathcal{Y}$ is the prediction function to be learned and $\ell: \mathcal{F}\times \mathcal{X}\times \mathcal{Y}\rightarrow \mathbb{R}^+$ is the loss function. 
Similarly generalization in supervised learning concerns the gap between the expected loss $\mathbb{E}[\ell(f, x, y)]$ and the empirical loss $\frac{1}{n}\sum_{i=1}^n \ell(f, x_i, y_i)$.

It is easy to find the correspondence between the episodes defined in Section \ref{sec:notation} and the samples $(x_i, y_i)$ in supervised learning. Just like supervised learning where $(x,y)\sim \mathcal{D}$, in (episodic) reinforcement learning ${\boldsymbol s}^i\sim \mathcal{D}_\pi$. Also the reward function $R$ in reinforcement learning is similar to the loss function $\ell$ in supervised learning. However, reinforcement learning is different because
\begin{itemize}
    \item In supervised learning, the sample distribution $\mathcal{D}$ is kept fixed, and the loss function $\ell\circ f$ changes as we choose different predictors $f$.
    \item In reinforcement learning, the reward function $R$ is kept fixed, but the sample distribution $\mathcal{D}_\pi$ changes as we choose different policy maps $\pi$.
\end{itemize}

As a consequence, the training procedure in reinforcement learning is also different. 
Popular methods such as REINFORCE \citep{Williams92}, Q-learning \citep{Sutton98}, and actor-critic methods \citep{mnih16} draw new states and episodes on the fly as the policy $\pi$ is being updated. 
That is, the distribution $\mathcal{D}_\pi$ from which episodes are drawn always changes during optimization. In contrast, in supervised learning we only update the predictor $f$ without affecting the underlying sample distribution $\mathcal{D}$.

\section{Intrinsic vs External Generalization Errors}

The generalization gap (\ref{obj:gap}) can be bounded 
\begin{align}
    \Phi\leq & \underbrace{\left|\frac{1}{n}\sum_{i=1}^n R({\boldsymbol s}^i)- \mathbb{E}_{{\boldsymbol s}\sim\mathcal{D}_{\hat{\pi}}} R({\boldsymbol s})\right|}_{\textrm{intrinsic}} \nonumber\\ 
    &+ \underbrace{\left|\mathbb{E}_{{\boldsymbol s}\sim\mathcal{D}_{\hat{\pi}}} R(\boldsymbol s)-\mathbb{E}_{\boldsymbol s\sim\mathcal{D}'_{\hat{\pi}}} R(\boldsymbol s)\right|}_{\textrm{external}}\label{eqn:error-term}
\end{align}
using the triangle inequality.
The first term in (\ref{eqn:error-term}) is the concentration error between the empirical reward and its expectation. 
Since it is caused by intrinsic randomness of the environment, we call it the \emph{intrinsic error}. 
Even if the test environment shares the same distribution with training, in the finite-sample scenario there is still a gap between training and testing. 
This is analogous to the overfitting problem studied in supervised learning. \citet{Zhang18} mainly focuses on this aspect of generalization. 
In particular, their randomness is carefully controlled in experiments to only come from the initial states $s_0\sim \mathcal{P}_0$. 


We call the second term in (\ref{eqn:error-term}) \emph{external error}, as it is caused by shifts of the distribution in the environment. 
For example, the transition distribution $\mathcal{P}$ or the initialization distribution $\mathcal{P}_0$ may get changed during testing, which leads to a different underlying episode distribution $\mathcal{D}'_\pi$. 
This is analogous to the transfer learning problem. For instance, generalization as in \citet{Cobbe18} is mostly external error since the number of levels used for training and testing are different even though the difficult level parameters are sampled from the same distribution. The setting in \citet{Packer18} covers both intrinsic and external errors.

\section{Why Intrinsic Generalization Error?}
If $\pi$ is fixed, by concentration of measures, as the number of episodes $n$ increases, the intrinsic error decreases roughly with $\frac{1}{\sqrt{n}}$.
For example, if the reward is bounded $|R({\boldsymbol s}^i)|\leq c/2$, by McDiarmid's bound, 
with probability at least $1-\delta$, 
\begin{align}
\left|\frac{1}{n}\sum_{i=1}^n R({\boldsymbol s}^i) - \mathbb{E}_{\boldsymbol s\sim \mathcal{D}}[R(\boldsymbol s)]\right|\leq c\sqrt{\frac{\log\frac{2}{\delta}}{2n}},\label{eqn:mcdiarmid}
\end{align}
where $c>0$.
Note the bound above also holds for the test samples if the distribution $\mathcal{D}$ is fixed and ${\boldsymbol s}_{test}\sim \mathcal{D}$.

For the population argument (\ref{obj:exp}), $\pi^\ast$ is defined deterministically since the value $\mathbb{E}_{\boldsymbol s\sim \mathcal{D}_\pi}R(\boldsymbol s)$ is a deterministic function of $\pi$. However, in the finite-sample case (\ref{obj:rl}), the policy map $\hat{\pi}$ is stochastic: it depends on the samples ${\boldsymbol s}^i$. As a consequence, the underlying distribution $\mathcal{D}_{\hat{\pi}}$ is not fixed. In that case, the expectation $\mathbb{E}_{\boldsymbol s\sim \mathcal{D}_{\hat{\pi}}}[R(\boldsymbol s)]$ in (\ref{eqn:mcdiarmid}) becomes a random variable so (\ref{eqn:mcdiarmid}) does not hold any more. 

One way of fixing the issue caused by random $\mathcal{D}_{\hat{\pi}}$ is to prove a bound that holds uniformly for all policies $\pi\in\Pi$. If $\pi$ is finite, by applying a union bound, it follows that:
\begin{lemma}
If $\Pi$ is finite, and $|R(\boldsymbol s)|\leq c/2$, then with probability at least $1-\delta$, for all $\pi\in\Pi$
\begin{align}
\left|\frac{1}{n}\sum_{i=1}^n R({\boldsymbol s}^i) - \mathbb{E}_{\boldsymbol s\sim \mathcal{D}_\pi}[R(\boldsymbol s)]\right|\leq c\sqrt{\frac{\log\frac{2|\Pi|}{\delta}}{2n}},
\end{align}
where $|\Pi|$ is the cardinality of $\Pi$.
\end{lemma}
Unfortunately in most of the applications, $\Pi$ is not finite. One difficulty in analyzing the intrinsic generalization error is that the policy changes during the optimization procedure. 
This leads to a change in the episode distribution $\mathcal{D}_\pi$. 
Usually $\pi$ is updated using episodes generated from some ``previous'' distributions, which are then used to generate new episodes.
In this case it is not easy to split episodes into a training and testing set, since during optimization samples always come from the updated policy distribution. 




\section{Reparameterization Trick}
The reparameterization trick has been popular in the optimization of deep networks \citep{Kingma14,Maddison17,Jang17,Tokui16} and used, e.g., for the purpose of optimization efficiency. 
In RL, suppose the objective (\ref{obj:exp}) is reparameterizable:
\begin{align}
    \mathbb{E}_{\boldsymbol s\sim\mathcal{D}_\pi} R(\boldsymbol s) = \mathbb{E}_{\xi\sim p(\xi)} R(\boldsymbol s(f(\xi,\pi))).\nonumber
\end{align}
Then under some weak assumptions
\begin{align}
    \nabla_\theta \mathbb{E}_{\boldsymbol s\sim\mathcal{D}_{\pi_\theta}} R(\boldsymbol s) = \nabla_\theta \left[\mathbb{E}_{\xi\sim p(\xi)} R(\boldsymbol s(f(\xi,\pi_\theta)))\right]\nonumber\\
    =  \mathbb{E}_{\xi\sim p(\xi)} \left[\nabla_\theta R(\boldsymbol s(f(\xi,\pi_\theta)))\right]
\end{align}
The reparameterization trick has already been used: for example, PGPE \citep{Ruckstie2010} uses policy reparameterization, and SVG \citep{Heess15} uses policy and environment dynamics reparameterization.
%
In this work, we will show the reparameterization trick can help to analyze the generalization gap. 
More precisely, we will show that since both $\mathcal{P}$ and $\mathcal{P}_0$ are fixed, even if they are unknown, as long as they satisfy some ``smoothness'' assumptions, we can provide theoretical guarantees on the test performance. 

\section{Reparameterized MDP}


%
We start our analysis with reparameterizing a Markov Decision Process with discrete states. 
We will give a general argument on reparameterizable RL in the next section. 
In this section we slightly abuse notation by letting $\mathcal{P}_0$ and $\mathcal{P}(s,a)$ denote $|S|$-dimensional probability vectors for multinomial distributions for initialization and transition respectively. 

One difficulty in the analysis of the generalization in reinforcement learning rises from the sampling steps in MDP where states are drawn from multinomial distributions specified by either $\mathcal{P}_0$ or $\mathcal{P}(s_t,a_t)$, because the sampling procedure does not explicitly connect the states and the distribution parameters. 
We can use standard Gumbel random variables $g\sim\exp(-g+\exp(-g))$ to reparameterize sampling and get a procedure equivalent to classical MDPs but with slightly different expressions, as shown in Algorithm \ref{alg:reparammdp}.
\begin{algorithm}[t]
\caption{Reparameterized MDP}
\label{alg:reparammdp}
\begin{algorithmic}
\STATE Initialization: Sample $g_{init},g_0, g_1, \dots, g_T\sim \mathcal{G}^{|S|}$.
$s_0=\arg\max \left(g_{init} + \log\mathcal{P}_0\right)$, $R=0$.
\FOR{$t$ in $0,\dots, T$}
\STATE $R = R + \gamma^t  r(s_t)$
\STATE $s_{t+1} =\arg\max \left(g_t + \log\mathcal{P}(s_t, \pi(s_t))\right)$
\ENDFOR
\STATE return $R$.
\end{algorithmic}
\end{algorithm}

In the reparameterized MDP procedure, $\mathcal{G}^{|S|}$ is an $|S|$-dimensional Gumbel distribution. $g_0, \dots, g_T$ are $|S|$-dimensional vectors with each entry being a Gumbel random variable. 
Also $g_0 + \log\mathcal{P}_0$ and $g_t + \log\mathcal{P}(s_t, a_t)$ are entry-wise vector sums, so they are both $|S|$-dimensional vectors. 
$\arg\max(v)$ returns the index of the maximum entry in the $|S|$-dimensional vector $v$. In the reparameterized MDP procedure shown above, the states $s_t$ are represented as an index in $\{1,2,\dots,|S|\}$.
After reparameterization, we may rewrite the RL objective (\ref{obj:rl}) as:\footnote{Again we abuse the notation by denoting ${\boldsymbol s}^i(f(g^i;\pi))$ as ${\boldsymbol s}^i(g^i;\pi)$.}
\begin{align}
     \hat{\pi} = \arg\max_{\pi\in \Pi, g^i\sim \mathcal{G}^{|S|T}}\frac{1}{n}\sum_{i=1}^n R({\boldsymbol s}^i(g^i;\pi)), \label{eqn:obj-rep}
\end{align}
where $g^i=[g_0^i, g_1^i,\dots, g_T^i]$, $g_t^i$ is an $|S|$-dimensional Gumbel random variable, and 
\begin{align}
R({\boldsymbol s}^i(g^i;\pi)) = \sum_{t=0}^T \gamma^t r(s_t^i({g_0^i, g_1^i, \dots, g_t^i};\pi))\label{eqn:emp-reward}
\end{align}
is the discounted return for one episode of length $T+1$. 

The reparameterized objective (\ref{eqn:obj-rep}) maximizes the empirical reward by varying the policy $\pi$. The distribution from which the random variables $g^i$ are drawn does not depend on the policy $\pi$ anymore, and the policy $\pi$ only affects the reward $R({\boldsymbol s}^i(g^i;\pi))$ through the states ${\boldsymbol s}^i$.

The objective (\ref{eqn:obj-rep}) is a discrete function due to the $\arg\max$ operator. One way to circumvent this is to use Gumbel softmax to approximate the $\arg\max$ operator \citep{Maddison17,Jang17}. 
If we denote $s$ as a one-hot vector in $\mathbb{R}^{|S|}$, and further relax the entries in $s$ to take positive values that sum up to one, we may use the softmax to approximate the $\arg\max$ operator. For instance, the reparametrized initial-state distribution becomes: 
\begin{align}
s_0 = \frac{\exp\{(g + \log\mathcal{P}_0)/\tau\}}{\|\exp\{(g + \log\mathcal{P}_0)/\tau\}\|_1},
\end{align}
where $g$ is an $|S|$-dimensional Gumbel random variable, $\mathcal{P}_0$ is an $|S|$-dimensional probability vector in multinomial distribution, and $\tau$ is a positive scalar.
As the temperature $\tau\rightarrow 0$, the softmax approaches $s=\arg\max \left(g + \log\mathcal{P}_0\right) \sim \mathcal{P}_0$ in terms of the one-hot vector representation.



\section{Reparameterizable RL}
In general, as long as the transition and initialization process can be reparameterized so that the environment parameters are separated from the random variables, the objective can always be reformulated so that the policy only affects the reward instead of the underlying distribution. 
The \textbf{reparameterizable RL} procedure is shown in Algorithm \ref{alg:reparamrl}. 

\begin{algorithm}[H]
\caption{Reparameterizzble RL}
\label{alg:reparamrl}
\begin{algorithmic}
\STATE Initialization: Sample $\xi_0, \xi_1, \dots, \xi_T$. $s_0= \mathcal{I}(\xi_0)$, $R=0$.
\FOR{$t$ in $0,\dots, T$}
\STATE $R = R + \gamma^t r(s_t)$
\STATE $s_{t+1} = \mathcal{T}(s_t, \pi(s_t), \xi_t)$
\ENDFOR
\STATE return $R$.
\end{algorithmic}
\end{algorithm}
In this procedure, $\xi$s are $d$-dimensional random variables but \emph{they are not necessarily sampled from the same distribution}.\footnote{They may also have different dimensions. In this work, without loss of generality, we assume the random variables have the same dimension $d$.}
In many scenarios they are treated as random noise. 
$\mathcal{I}:\mathbb{R}^{d}\rightarrow \mathbb{R}^{|S|}$ is the \emph{initialization function}. 
During initialization, the random variable $\xi_0$ is taken as input and the output is an initial state $s_0$. 
The transition function $\mathcal{T}: \mathbb{R}^{|S|}\times \mathbb{R}^{|A|}\times \mathbb{R}^{d}\rightarrow \mathbb{R}^{|S|}$, takes the current state $s_t$, the action produced by the policy $\pi(s_t)$, and a random variable $\xi_t$ to produce the next state $s_{t+1}$. 

In reparameterizable RL, the peripheral random variables $\xi_0,\xi_1,\dots,\xi_T$ can be sampled before the episode is generated. In this way, the randomness is decoupled from the policy function, and as the policy $\pi$ gets updated, the episodes can be computed deterministically.

The class of reparamterizable RL problems includes those whose initial state, transition, reward and optimal policy distribution can be reparameterized. 
Generally, a distribution can be reparameterized, e.g., if it has a tractable inverse CDF, is a composition of reparameterizable distributions \citep{Kingma14}, or is a limit of smooth approximators \citep{Maddison17,Jang17}.
Reparametrizable RL settings include LQR \citep{lewis1995optimal} and physical systems (e.g., robotics) where the dynamics are given by stochastic partial differential equations (PDE) with reparameterizable components over continuous state-action spaces. 

\section{Main Result}

For \textbf{reparameterizable RL}, if the environments and the policy are ``smooth'', we can control the error between the expected and the empirical reward. 
In particular, the assumptions we make are\footnote{$\|\cdot\|$ is the $L_2$ norm, and $\theta\in\mathbb{R}^m$.}

\begin{assumption}\label{assump:lipschitzT}
$\mathcal{T}(s,a): \mathbb{R}^{|S|}\times \mathcal{R}^{|A|}\rightarrow \mathbb{R}^{|S|}$ is $L_{t1}$-Lipschitz in terms of the first variable $s$, and $L_{t2}$-Lipschitz in terms of the second variable $a$. That is, $\forall x, x', y, y', z$, 
\begin{align}
\|\mathcal{T}(x,y,z)-\mathcal{T}(x',y,z)\|\leq L_{t1}\|x-x'\|,\nonumber\\
\|\mathcal{T}(x,y,z)-\mathcal{T}(x,y',z)\|\leq L_{t2}\|y-y'\|.\nonumber
\end{align}
\end{assumption}
\begin{assumption}\label{assump:lipschitzP}
The policy is parameterized as $\pi(s;\theta):\mathbb{R}^{|S|}\times \mathbb{R}^m\rightarrow \mathbb{R}^{|A|}$, and $\pi(s;\theta)$ is $L_{\pi1}$-Lipschitz in terms of the states, and $L_{\pi2}$-Lipschitz in terms of the parameter $\theta\in\mathbb{R}^m$, that is, $\forall s, s', \theta, \theta'$
\begin{align}
    \|\pi(s;\theta)-\pi(s';\theta)\|\leq L_{\pi1} \|s-s'\|,\nonumber\\
    \|\pi(s;\theta)-\pi(s;\theta')\|\leq L_{\pi2} \|\theta-\theta'\|.\nonumber
\end{align}
\end{assumption}
\begin{assumption}\label{assump:lipschitzR}
The reward $r(s):\mathbb{R}^{|S|}\rightarrow \mathbb{R}$ is $L_r$-Lipschitz:
\begin{align}
    |r(s')-r(s)|\leq L_r\|s'-s\|.\nonumber
\end{align}
\end{assumption}
If assumptions (\ref{assump:lipschitzT}) (\ref{assump:lipschitzP}) and (\ref{assump:lipschitzR}) hold, we have the following:
\begin{theorem} \label{thm: main}
In \textbf{reparameterizable RL}, suppose the transition $\mathcal{T}'$ in the test environment satisfies $\forall x, y,z, \|(\mathcal{T}'-\mathcal{T})(x,y,z)\|\leq \zeta$, and suppose the initialization function $\mathcal{I}'$ in the test environment satisfies $\forall \xi, \|(\mathcal{I}'-\mathcal{I})(\xi)\|\leq \epsilon$. If assumptions (\ref{assump:lipschitzT}), (\ref{assump:lipschitzP}) and (\ref{assump:lipschitzR}) hold, the peripheral random variables $\xi^i$ for each episode are i.i.d., and the reward is bounded $|R(\boldsymbol s)|\leq c/2$, then with probability at least $1-\delta$, for all policies $\pi\in\Pi$:
\begin{align}
    &|\mathbb{E}_\xi[R(\boldsymbol s(\xi;\pi, \mathcal{T}', \mathcal{I}'))] - \frac{1}{n}\sum_i R(\boldsymbol s(\xi^i;\pi,\mathcal{T}, \mathcal{I}))|\nonumber\\
    &\leq  Rad(R_{\pi,\mathcal{T},\mathcal{I}}) + L_r \zeta\sum_{t=0}^T \gamma^{t} \frac{\nu^t-1}{\nu-1} + L_r \epsilon\sum_{t=0}^{T} \gamma^{t}  \nu^t\nonumber\\&+ O\left(c\sqrt{\frac{\log (1/\delta)}{n}}\right), \nonumber
\end{align}
where $\nu=L_{t1}+L_{t2}L_{\pi1}$, and  $Rad(R_{\pi,\mathcal{T},\mathcal{I}}) = \mathbb{E}_{\xi}\mathbb{E}_\sigma\left[\sup_\pi \frac{1}{n}\sum_{i=1}^n \sigma_i R({\boldsymbol s}^i(\xi^i;\pi,\mathcal{T}, \mathcal{I}))\right]$ is the Rademacher complexity of $R(\boldsymbol s(\xi;\pi,\mathcal{T},\mathcal{I}))$ under the training transition $\mathcal{T}$, the training initialization $\mathcal{I}$, and $n$ is the number if training episodes.
\end{theorem}

Note the i.i.d. assumption on the peripheral variables $\xi^i$ is across episodes. Within the same episode, there could be correlations among the $\xi_t^i$s at different time steps. 

Similar arguments can also be made when the transition $T'$ in the test environment stays the same as $\mathcal{T}$, but the initialization $\mathcal{I}'$ is different from $\mathcal{I}$. In the following sections we will bound the intrinsic and external errors respectively.
\section{Bounding Intrinsic Generalization Error}\label{sec:intrinc}
After reparameterization, the objective (\ref{eqn:obj-rep}) is essentially the same as an empirical risk minimization problem in the supervised learning scenario. 
According to classical learning theory, the following lemma is straight-forward \citep{Shalev-Shwartz2014}:
\begin{lemma}\label{lemma:rad}
If the reward is bounded, $|R(\boldsymbol s)|\leq c/2, c>0$, and $g^i\sim G^{|S|\times T}$ are i.i.d. for each episode, with probability at least $1-\delta$, for $\forall \pi\in \Pi$:
\begin{align}
    |\mathbb{E}_{g\sim \mathcal{G}^{|S|\times T}}[R(\boldsymbol s(g;\pi))]-\frac{1}{n}\sum_i R({\boldsymbol s}^i(g^i;\pi))| \nonumber\\
    \leq  Rad(R_\pi)+ O\left(c\sqrt{\frac{\log(1/\delta)}{n}}\right),\label{eqn:rad}
\end{align}
where $Rad(R_\pi) = \mathbb{E}_{g}\mathbb{E}_\sigma\left[\sup_\pi \frac{1}{n}\sum_{i=1}^n \sigma_i R({\boldsymbol s}^i(g^i;\pi))\right]$ is the Rademacher complexity of $R(\boldsymbol s(g;\pi))$.
\end{lemma}
The bound (\ref{eqn:rad}) holds uniformly for all $\pi\in \Pi$, so it also holds for $\hat{\pi}$. 
Unfortunately, in MDPs $Rad(R_\pi)$ is hard to control, mainly due to the recursive $\arg\max$ in the representation of the states, $s_{t+1} =\arg\max \left(g_t + \log\mathcal{P}(s_t, \pi(s_t))\right)$. 

On the other hand, for general reparameterizable RL we may control the intrinsic generalization gap by assuming some ``smoothness'' conditions on the transitions $\mathcal{T}$, as well as the policy $\pi$. 
In particular, it is straight-forward to prove that the empirical return $R$ is ``smooth'' if the transitions and policies are all Lipschitz.
\begin{lemma}
For \textbf{reparameterizable RL}, given assumptions \ref{assump:lipschitzT}, \ref{assump:lipschitzP}, and \ref{assump:lipschitzR}, the empirical return $R$ defined in (\ref{eqn:emp-reward}), as a function of the parameter $\theta$, has a Lipschitz constant of 
\begin{align}
\beta = L_rL_{t_2}L_{\pi2} \sum_{t=0}^T \gamma^t \frac{\nu^t-1}{\nu-1},\label{eqn:reward-const}
\end{align}
where $\nu=L_{t1}+L_{t2}L_{\pi1}$.
\end{lemma}
Also, if the number of parameters $m$ in $\pi(\theta)$ is bounded, then the Rademacher complexity $Rad(R_\pi)$ in Lemma 2 can be controlled \citep{Vaart98,Bartlett13}.

\begin{lemma}
For \textbf{reparameterizable RL}, given assumptions \ref{assump:lipschitzT}, \ref{assump:lipschitzP}, and \ref{assump:lipschitzR}, if the parameters $\theta\in \mathbb{R}^m$ is bounded such that $\|\theta\|\leq 1$, and the function class of the reparameterized reward $\mathcal{R}$ is closed under negations, then the Rademacher complexity $Rad(R_\pi)$ is bounded by
\begin{align}
    Rad(R_\pi) = O\left(\beta\sqrt{\frac{m}{n}}\right)\label{eqn:rad-bound}
\end{align}
where $\beta$ is the Lipschitz constant defined in (\ref{eqn:reward-const}), and $n$ is the number of episodes.
\end{lemma}
In the context of deep learning, deep neural networks are over-parameterized models that have proven to work well in many applications. 
However, the bound above does not explain why over-parameterized models also generalize well since the Rademacher complexity bound (\ref{eqn:rad-bound}) can be extremely large as $m$ grows. 
To ameliorate this issue, recently \citet{Arora18} proposed a compression approach that compresses a neural network to a smaller one with fewer parameters but has roughly the same training errors. 
Whether this also applies to reparameterizable RL is yet to be proven. There are also trajectory-based techniques proposed to sharpen the generalization bound \citep{Yuanzhi17,Zhu18,Arora19,Cao19}.
\subsection{PAC-Bayes Bound on Reparameterizable RL}
We can also analyze the Rademacher complexity of the empirical return by making a slightly different assumption on the policy. 
Suppose $\pi$ is parameterized as $\pi(\theta)$, and $\theta$ is sampled from some posterior distribution $\theta \sim \mathcal{Q}$. 
According to the PAC-Bayes theorem \citep{McAllester1998,Mcallester03,Neyshabur2017a,Langford2002}:
\begin{lemma}
Given a ``prior'' distribution $\mathcal{D}_0$, with probability at least $1-\delta$ over the draw of $n$ episodes, $\forall \mathcal{Q}$:
\begin{align}
\mathbb{E}_g[\mathrm{R}_{\theta\sim \mathcal{Q}}(g)]
\geq &\frac{1}{n}\sum_i \mathrm{R}_{\theta \sim\mathcal{Q}}(g^i) \nonumber\\
&- 2\sqrt{\frac{2(KL(\mathcal{Q}||\mathcal{D}_0) + \log \frac{2n}{\delta})}{n-1}},\label{eqn:pacbayes} \\ 
R_{\theta\sim\mathcal{Q}}(g^i) &= \mathbb{E}_{\theta\sim\mathcal{Q}}\left[R({\boldsymbol s}^i(g^i;\pi(\theta) ))\right] \nonumber\\
&= \mathbb{E}_{\theta\sim\mathcal{Q}}\left[\sum_{t=0}^T \gamma^t r(s_t^i(g^i;\pi(\theta)))\right],
\end{align}
where $R_{\theta\sim\mathcal{Q}}(g)$
is the expected ``Bayesian'' reward.
\end{lemma}
The bound (\ref{eqn:pacbayes}) holds for all posterior $\mathcal{Q}$. 
In particular it holds if $\mathcal{Q}$ is $\theta + u$ where $\theta$ could be any solution provided by empirical return maximization, and $u$ is a perturbation, e.g., zero-centered uniform or Gaussian distribution. 
This suggests maximizing a perturbed objective instead may lead to better generalization performance, which has already been observed empirically \citep{WangJ18}.

The tricky part about perturbing the policy is choosing the level of noise. 
Suppose there is an empirical reward optimizer $\pi(\hat{\theta})$. When the noise level is small, the first term in (\ref{eqn:pacbayes}) is large, but the second term may also be large since the posterior $\mathcal{Q}$ is too focused on $\hat{\theta}$ but the ``prior'' $\mathcal{D}_0$ cannot depend on $\hat{\theta}$, and vice versa. 
On the other hand, if the reward function is ``nice'', e.g., if some ``smoothness'' assumption holds in a local neighborhood of $\hat{\theta}$, 
then one can prove the optimal noise level roughly scales inversely as the square root of the local Hessian diagonals \citep{Wang18}. 


\section{Bounding External Generalization Error}\label{sec:external}

Another source of generalization error in RL comes from the change of environment. For example, in an MDP $(\mathcal{S}, \mathcal{A}, \mathcal{P}, r, \mathcal{P}_0)$, the transition probability $\mathcal{P}$ or the initialization distribution $\mathcal{P}_0$ is different in the test environment. 
\citet{Cobbe18} and \citet{Packer18} show that as the distribution of the environment varies the gap between the training and testing could be huge. 

Indeed if the test distribution is drastically different from the training environment, there is no guarantee the performance of the same model could possibly work for testing. On the other hand, if the test distribution $\mathcal{D}'$ is not too far away from the training distribution $\mathcal{D}$ then the test error can still be controlled. For example, for supervised learning, \citet{Mohri12} prove the expected loss of a drifting distribution is also bounded. In addition to Rademacher complexity and a concentration tail, there is one more term in the gap that measures the discrepancy between the training and testing distribution. 

For reparameterizable RL, since the environment parameters are lifted into the reward function in the reformulated objective (\ref{eqn:obj-rep}), the analysis becomes easier. 
For MDPs, a small change in environment could cause large difference in the reward since $\arg\max$ is not continuous. 
However if the transition function is ``smooth'', the expected reward in the new environment can also be controlled. 
e.g., if we assume the transition function $\mathcal{T}$, the reward function $r$, as well as the policy function $\pi$ are all Lipschitz, as in section \ref{sec:intrinc}.

If the transition function $\mathcal{T}$ is the same in the test environment and the only difference is the initialization, we can prove the following lemma:
\begin{lemma}\label{lemma:initchange}
In \textbf{reparameterizable RL}, suppose the initialization function $\mathcal{I}'$ in the test environment satisfies $\forall \xi, \|(\mathcal{I}'-\mathcal{I})(\xi)\|\leq \zeta$ for $\zeta > 0$, and the transition function $\mathcal{T}$ in the test environment is the same as training. If assumptions (\ref{assump:lipschitzT}), (\ref{assump:lipschitzP}), and (\ref{assump:lipschitzR}) hold, then:
\begin{align}
    &|\mathbb{E}_\xi[R(\boldsymbol s(\xi;\mathcal{I}'))] - \mathbb{E}_\xi[R(\boldsymbol s(\xi;\mathcal{I}))]|\nonumber\\ &\leq L_r \zeta\sum_{t=0}^{T} \gamma^{t}  (L_{t1}+L_{t2}L_{\pi1})^t \label{eqn:init-shift}
\end{align}
\end{lemma}
Lemma \ref{lemma:initchange} means that if the initialization in the test environment is not too different from the training one, and if the transition, policy and reward functions are smooth, then the expected reward in the test environment won't deviate from that of training too much. 

The other possible environment change is that the test initialization $\mathcal{I}$ stays the same but the transition changes from the training transition $\mathcal{T}$ to $\mathcal{T}'$. Similar to before, we have:
\begin{lemma}\label{lemma:external}
In \textbf{reparameterizable RL}, suppose the transition $\mathcal{T}'$ in the test environment satisfies $\forall x, y,z, \|(\mathcal{T}'-\mathcal{T})(x,y,z)\|\leq \zeta$, and the initialization $\mathcal{I}$ in the test environment is the same as training. If assumptions (\ref{assump:lipschitzT}), (\ref{assump:lipschitzP}) and (\ref{assump:lipschitzR}) hold then 
\begin{align}
    &|\mathbb{E}_\xi[R(\boldsymbol s(\xi;\mathcal{T}'))] - \mathbb{E}_\xi[R(\boldsymbol s(\xi;\mathcal{T}))]|\nonumber\\ &\leq L_r \zeta\sum_{t=0}^T \gamma^{t} \frac{\nu^t-1}{\nu-1} \label{eqn:transit-shift}
\end{align}
where $\nu=L_{t1}+L_{t2}L_{\pi1}$.
\end{lemma}
The difference between (\ref{eqn:transit-shift}) and (\ref{eqn:init-shift}) is that the change $\zeta$ in transition $\mathcal{T}$ is further enlarged during an episode: as long as $\nu>1$, the gap in (\ref{eqn:transit-shift}) is larger and can become huge as the length $T$ of the episode increases.
\section{Simulation}
\begin{table}[]
\caption{Intrinsic Gap versus Smoothness}\label{table:intrinsic}
\centering
\begin{tabular}{ccccc}
\textbf{Temperature} &\multicolumn{2}{c}{\textbf{Policy}}  & \textbf{State} &  \textbf{Action}  \\
$\tau$ & Gap  & $\frac{1}{\tau}\Pi_l \|\hat{\theta}^l\|_F$ & Gap  & Gap  \\
\hline
$ 0.001  $ & $0.554 $ & $2.20\cdot 10^{6} $ & $0.632 $ & $0.612 $ \\
$ 0.01  $ & $0.494 $ & $4.46\cdot 10^{5} $ & $0.632 $ & $0.608 $ \\
$ 0.1 $ & $0.482 $ & $1.74\cdot 10^{5} $ & $0.633 $ & $0.603$ \\
$ 1 $ & $0.478 $ & $8.83\cdot 10^{4} $ & $0.598 $ & $0.598 $ \\
$ 10 $ & $0.479 $ & $5.06\cdot 10^{4} $ & $0.588 $ & $0.594 $ \\
$ 100 $ & $0.468 $ & $4.77\cdot 10^{4} $ & $0.581 $ & $0.594 $ \\
$ 1000 $ & $0.471 $ & $3.29\cdot 10^{4} $ & $0.590 $ & $0.594 $ \\
\end{tabular}
\end{table}
We now present empirical measurements in simulations to verify some claims made in section \ref{sec:intrinc} and \ref{sec:external}. 
The bound (\ref{eqn:rad-bound}) suggests the gap between the expected reward and the empirical reward is related to the Lipschitz constant $\beta$ of $R$, which according to equation (\ref{eqn:reward-const}) is related to the Lipschitz constant of a series of functions including $\pi$, $\mathcal{T}$, and $r$. 
\subsection{Intrinsic Generalization Gap}
In (\ref{eqn:reward-const}), as the length of the episode $T$ increases, the dominating factors in $\beta$ are $L_{t1}$, $L_{t2}$ and $L_{\pi1}$. 
Our first simulation fixes the environment and verifies $L_{\pi}$. 
In the simulation, we assume the initialization $\mathcal{I}$ and the transition $\mathcal{T}$ are all known and fixed. 
$\mathcal{I}$ is an identity function, and $\xi_0\in \mathbb{R}^{|S|}$ is a vector of i.i.d. uniformly distributed random variables: $\xi_0[k]\sim U[0,1], \forall k\in 1,\dots |S|$. The transit function is $\mathcal{T}(s, a, \xi) = sT_1 + aT_2 + \xi T_3$,
where $s\in\mathbb{R}^{|S|}$, $a\in \mathbb{R}^{|A|}$, $\xi\in\mathbb{R}^2$ are row vectors, and $T_1\in \mathbb{R}^{|S|\times|S|}$, $T_2\in\mathbb{R}^{|A|\times |S|}$, and $T_3\in\mathbb{R}^{2\times|S|}$ are matrices used to project the states, actions, and noise respectively.
$T_1$, $T_2$, and $T_3$ are randomly generated and then kept fixed during the experiment. 
We use $\gamma=1$ as the discounting constant throughout.

The policy $\pi(s, \theta)$ is modeled using a multiple layer perceptron (MLP) with rectified linear as the activation. The last layer of MLP is a linear layer followed by a softmax function with temperature: $q(x[k];\tau) = \frac{\exp{\frac{x[k]}{\tau}}}{\sum_k \exp{\frac{x[k]}{\tau}}}$.

By varying the temperature $\tau$ we are able to control the Lipschitz constant of the policy class $L_{\pi1}$ and $L_{\pi2}$ if we assume the bound on the parameters $\|\theta\|\leq B$ is unchanged.

We set the length of the episode $T=128$, and randomly sample $\xi_0,\xi_1,\dots, \xi_T$ for $n=128$ training and testing episodes. Then we use the same random noise to evaluate a series of policy classes with different temperatures $\tau\in\{0.001,0.01, 0.1, 1,10, 100, 1000\}$. 

Since we assume $\mathcal{I}$ and $\mathcal{T}$ are known, during training the computation graph is complete. Hence we can directly optimize the coefficients $\theta$ in $\pi(s;\theta)$ just as in supervised learning.\footnote{In real applications this is not doable since $\mathcal{T}$ and $\mathcal{I}$ are unknown. Here we assume they are known just to investigate the generalization gap.} 
We use Adam \citep{KingmaB14} to optimize with initial learning rates $10^{-2}$ and $10^{-3}$. When the reward stops increasing we halved the learning rate. 
and analyze the gap between the average training and testing reward.

First, we observe the gap is affected by the optimization procedure. 
For example, different learning rates can lead to different local optima, even if we decrease the learning rate by half when the reward does not increase. 
Second, even if we know the environment $\mathcal{I}$ and $\mathcal{T}$, so that we can optimize the policy $\pi(s;\theta)$ directly, we still experience unstable learning just like other RL algorithms. 
This suggests that the unstableness of the RL algorithms may not rise from the estimation of the environment for the model based algorithms such as A2C and A3C \citep{mnih16}, since even if we know the environment the learning is still unstable.

Given the unstable training procedure, for each trial we ran the training for $1024$ epochs with learning rate of 1e-2 and 1e-3, and the one with higher training reward at the last epoch is used for reporting. 
Ideally as we vary $\tau$, the Lipschitz constant for the function class $\pi\in\Pi$ is changed accordingly given the assumption $\|\theta\|\leq B$. 
However, it is unclear if $B$ is changed or not for different configurations. 
After all, the assumption that the parameters are bounded is artificial. 
To ameliorate this defect we also check the metric $\frac{1}{\tau}\Pi_l \|\theta^l\|_F$, where $\theta^l$ is the weight matrix of the $l$th layer of MLP. 
In our experiment there is no bias term in the linear layers in MLP, so $\frac{1}{\tau}\Pi_l \|\hat{\theta}^l\|_F$ can be used as a metric on the Lipschitz constant $L_{\pi1}$ at the solution point $\hat{\theta}$. 
%
%
We also vary the smoothness in the transition function a a function of states ($T_1$), and actions ($T_2$), by applying softmax with different temperatures $\tau$ to the singular values of the randomly generated matrix.



 

Table \ref{table:intrinsic} shows the average generalization gap roughly decreases as $\tau$ decreases. 
The metric $\frac{1}{\tau}\Pi_l \|\hat{\theta}^l\|_F$ also decreases similarly as the average gap. 
In particular, the 2nd and 3rd column shows the average gap as the policy becomes ``smoother". 
The 4th column shows, if we fix the policy-$\tau$ as well as setting $T_2=1$, the generalization gap decreases as we increase the transition-$\tau$ for $T_1$ (states). 
Similarly the last column is the gap as the transition-$\tau$ for actions ($T_2$) varies.
In Table \ref{table:param} the environment is fixed and for each parameter configuration the gap is averaged from $100$ trials with randomly initialized and then optimized policies.

\begin{table}[t]
\centering
\begin{tabular}{lccccc}
    \textbf{Params} & 65.6k & 131.3k & 263.2k & 583.4k & 1.1m \\
    \textbf{Gap} & 0.204 & 0.183  & 0.214 & 0.336 & 0.418\\
\end{tabular}
\vspace{-5pt}
\caption{Empirical gap vs \#policy params.}
\label{table:param}
\vspace{10pt}
\begin{tabular}{lcccc}
\textbf{$\zeta$ in $\mathcal{I}$} & 1 & 10 & 100 & 1,000 \\ 
\textbf{Gap} & 0.481 & 0.477 & 0.659 & 0.532
\end{tabular}
\vspace{-5pt}
\centering
\caption{Empirical generalization gap vs shift in initialization.}
\label{table:external:init}
\vspace{10pt}
%
\begin{tabular}{lcccc}
    \textbf{$\zeta$ in $\mathcal{T}$} & 1 & 10 & 100 & 1,000 \\  
    \textbf{Gap} & 11 & 451 & 8,260 & 73,300
\end{tabular}
\vspace{-5pt}
\caption{Empirical generalization gap vs shift in transition.} 
\label{table:external:tr}
\end{table}

\subsection{External Generalization Gap}
To measure the external generalization gap, we vary the transition $\mathcal{T}$ as well as the initialization $\mathcal{I}$ in the test environment. 
For that, we add a vector of Rademacher random variables $\Delta$ to $\mathcal{I}$ or $\mathcal{T}$, with $\|\Delta\|=\zeta$. 
We adjust the level of noise $\delta$ in the simulation and report the change of the average gap in Table \ref{table:external:init} and Table \ref{table:external:tr}. It is not surprising that the change $\Delta_T$ in transition $\mathcal{T}$ leads to a higher generalization gap since the impact from $\Delta_T$ is accumulated across time steps. Indeed if we compare the bound (\ref{eqn:transit-shift}) and (\ref{eqn:init-shift}), when $\gamma=1$ as long as $\nu>1$, the gap in (\ref{eqn:transit-shift}) is larger.

\section{Discussion and Future Work}

Even though a variety of distributions, discrete or continuous, can be reparameterized, and we have shown that the classical MDP with discrete states is reparameterizable, it is not clear in general under which conditions reinforcement learning problems are reparameterizable. 
Classifying particular cases where RL is not reparameterizable is an interesting direction for future work.
Second, the transitions of discrete MDPs are inherently non-smooth, so Theorem \ref{thm: main} does not apply. 
In this case, the PAC-Bayes bound can be applied, but this requires a totally different framework. 
It will be interesting to see if there is a ``Bayesian'' version of Theorem \ref{thm: main}. 
Finally, our analysis only covers ``on-policy'' RL. Studying generalization for ``off-policy'' RL remains an interesting future topic. 

\bibliography{library}

\begin{thebibliography}{62}
\providecommand{\natexlab}[1]{#1}
\providecommand{\url}[1]{\texttt{#1}}
\expandafter\ifx\csname urlstyle\endcsname\relax
  \providecommand{\doi}[1]{doi: #1}\else
  \providecommand{\doi}{doi: \begingroup \urlstyle{rm}\Url}\fi

\bibitem[Agarwal et~al.(2014)Agarwal, Hsu, Kale, Langford, Li, and
  Schapire]{Agarwal14}
Agarwal, A., Hsu, D., Kale, S., Langford, J., Li, L., and Schapire, R.
\newblock Taming the monster: A fast and simple algorithm for contextual
  bandits.
\newblock \emph{International Conference on Machine Learning}, 2014.

\bibitem[Allen{-}Zhu et~al.(2018)Allen{-}Zhu, Li, and Liang]{Zhu18}
Allen{-}Zhu, Z., Li, Y., and Liang, Y.
\newblock Learning and generalization in overparameterized neural networks,
  going beyond two layers.
\newblock \emph{CoRR}, abs/1811.04918, 2018.

\bibitem[Arora et~al.(2018)Arora, Ge, Neyshabur, and Zhang]{Arora18}
Arora, S., Ge, R., Neyshabur, B., and Zhang, Y.
\newblock Stronger generalization bounds for deep nets via a compression
  approach.
\newblock \emph{International Conference on Machine Learning}, 2018.

\bibitem[Arora et~al.(2019)Arora, Du, Hu, Li, and Wang]{Arora19}
Arora, S., Du, S.~S., Hu, W., Li, Z., and Wang, R.
\newblock Fine-grained analysis of optimization and generalization for
  overparameterized two-layer neural networks.
\newblock \emph{International Conference on Machine Learning}, 2019.

\bibitem[Auer et~al.(2002)Auer, Cesa-Bianchi, and Fischer]{Auer02}
Auer, P., Cesa-Bianchi, N., and Fischer, P.
\newblock Finite-time analysis of the multiarmed bandit problem.
\newblock \emph{Maching Learning}, 2002.

\bibitem[Auer et~al.(2009)Auer, Jaksch, and Ortner]{Auer09}
Auer, P., Jaksch, T., and Ortner, R.
\newblock Near-optimal regret bounds for reinforcement learning.
\newblock \emph{Advances in Neural Information Processing Systems 21}, 2009.

\bibitem[Azar et~al.(2017)Azar, Osband, and Munos]{Azar17}
Azar, M.~G., Osband, I., and Munos, R.
\newblock Minimax regret bounds for reinforcement learning.
\newblock \emph{International Conference on Machine Learning}, 2017.

\bibitem[Baker et~al.(2017)Baker, Gupta, Naik, and Raskar]{Baker17}
Baker, B., Gupta, O., Naik, N., and Raskar, R.
\newblock Designing neural network architectures using reinforcement learning.
\newblock 2017.

\bibitem[Bartlett(2013)]{Bartlett13}
Bartlett, P.
\newblock Lecture notes on theoretical statistics.
\newblock 2013.

\bibitem[Beygelzimer et~al.(2011)Beygelzimer, Langford, Li, Reyzin, and
  Schapire]{beygelzimer11}
Beygelzimer, A., Langford, J., Li, L., Reyzin, L., and Schapire, R.
\newblock Contextual bandit algorithms with supervised learning guarantees.
\newblock \emph{Proceedings of the Fourteenth International Conference on
  Artificial Intelligence and Statistics}, 2011.

\bibitem[Cao \& Gu(2019)Cao and Gu]{Cao19}
Cao, Y. and Gu, Q.
\newblock A generalization theory of gradient descent for learning
  over-parameterized deep relu networks.
\newblock \emph{CoRR}, abs/1902.01384, 2019.

\bibitem[Chakravorty \& Hyland(2003)Chakravorty and Hyland]{Chakravorty03}
Chakravorty, S. and Hyland, D.~C.
\newblock Minimax reinforcement learning.
\newblock \emph{American Institute of Aeronautics and Astronautic}, 2003.

\bibitem[Cobbe et~al.(2018)Cobbe, Klimov, Hesse, Kim, and Schulman]{Cobbe18}
Cobbe, K., Klimov, O., Hesse, C., Kim, T., and Schulman, J.
\newblock Quantifying generalization in reinforcement learning.
\newblock \emph{CoRR}, 2018.
\newblock URL \url{http://arxiv.org/abs/1812.02341}.

\bibitem[Dann et~al.(2017)Dann, Lattimore, and Brunskill]{Dann17}
Dann, C., Lattimore, T., and Brunskill, E.
\newblock Unifying pac and regret: Uniform pac bounds for episodic
  reinforcement learning.
\newblock \emph{International Conference on Neural Information Processing
  Systems (NIPS)}, 2017.

\bibitem[Farebrother et~al.(2018)Farebrother, Machado, and
  Bowling]{Farebrother18}
Farebrother, J., Machado, M.~C., and Bowling, M.
\newblock Generalization and regularization in dqn.
\newblock \emph{CoRR}, 2018.
\newblock URL \url{https://arxiv.org/abs/1810.00123}.

\bibitem[Heess et~al.(2015)Heess, Wayne, Silver, Lillicrap, Erez, and
  Tassa]{Heess15}
Heess, N., Wayne, G., Silver, D., Lillicrap, T., Erez, T., and Tassa, Y.
\newblock Learning continuous control policies by stochastic value gradients.
\newblock \emph{Advances in Neural Information Processing Systems}, 2015.

\bibitem[Jang et~al.(2017)Jang, Gu, and Poole]{Jang17}
Jang, E., Gu, S., and Poole, B.
\newblock Categorical reparameterization by gumbel-softmax.
\newblock \emph{International Conference on Learning Representations}, 2017.

\bibitem[Jiang et~al.(2017)Jiang, Krishnamurthy, Agarwal, Langford, and
  Schapire]{Jiang17}
Jiang, N., Krishnamurthy, A., Agarwal, A., Langford, J., and Schapire, R.~E.
\newblock Contextual decision processes with low {B}ellman rank are
  {PAC}-learnable.
\newblock \emph{International Conference on Machine Learning}, 2017.

\bibitem[Justesen et~al.(2018)Justesen, Torrado, Bontrager, Khalifa, Togelius,
  and Risi]{Justesen18}
Justesen, N., Torrado, R.~R., Bontrager, P., Khalifa, A., Togelius, J., and
  Risi, S.
\newblock Illuminating generalization in deep reinforcement learning through
  procedural level generation.
\newblock \emph{NeurIPS Deep RL Workshop}, 2018.

\bibitem[Kearns \& Singh(2002)Kearns and Singh]{Kearns2002}
Kearns, M. and Singh, S.
\newblock Near-optimal reinforcement learning in polynomial time.
\newblock \emph{Mache Learning}, 2002.

\bibitem[Kingma \& Ba(2015)Kingma and Ba]{KingmaB14}
Kingma, D.~P. and Ba, J.
\newblock Adam: {A} method for stochastic optimization.
\newblock \emph{International Conference on Learning Representations}, 2015.

\bibitem[Kingma \& Welling(2014)Kingma and Welling]{Kingma14}
Kingma, D.~P. and Welling, M.
\newblock Auto-encoding variational bayes.
\newblock \emph{International Conference on Learning Representations}, 2014.

\bibitem[Kober et~al.(2013)Kober, Bagnell, and Peters]{Kober2013}
Kober, J., Bagnell, J.~A., and Peters, J.
\newblock Reinforcement learning in robotics: A survey.
\newblock \emph{International Journal of Robotic Research}, 2013.

\bibitem[Langford \& Shawe-Taylor(2002)Langford and Shawe-Taylor]{Langford2002}
Langford, J. and Shawe-Taylor, J.
\newblock Pac-bayes \& margins.
\newblock \emph{International Conference on Neural Information Processing
  Systems (NIPS)}, 2002.

\bibitem[Laroche(2017)]{Laroche17}
Laroche, R.
\newblock Transfer reinforcement learning with shared dynamics.
\newblock 2017.

\bibitem[Lattimore \& Hutter(2014)Lattimore and Hutter]{Lattimore14}
Lattimore, T. and Hutter, M.
\newblock Near-optimal {PAC} bounds for discounted mdps.
\newblock \emph{Theoretical Computer Science}, 2014.

\bibitem[Lazaric(2012)]{Lazaric12}
Lazaric, A.
\newblock Transfer in reinforcement learning: a framework and a survey.
\newblock \emph{Reinforcement Learning - State of the Art, Springer}, 2012.

\bibitem[Lewis et~al.(1995)Lewis, Syrmos, and Syrmos]{lewis1995optimal}
Lewis, F., Syrmos, V., and Syrmos, V.
\newblock \emph{Optimal Control}.
\newblock A Wiley-interscience publication. Wiley, 1995.
\newblock ISBN 9780471033783.
\newblock URL \url{https://books.google.com/books?id=jkD37elP6NIC}.

\bibitem[Li et~al.(2010)Li, Chu, Langford, and Schapire]{Li10}
Li, L., Chu, W., Langford, J., and Schapire, R.~E.
\newblock A contextual-bandit approach to personalized news article
  recommendation.
\newblock \emph{Proceedings of the 19th International Conference on World Wide
  Web}, 2010.

\bibitem[Li et~al.(2018)Li, Ma, and Zhang]{Yuanzhi17}
Li, Y., Ma, T., and Zhang, H.
\newblock Algorithmic regularization in over-parameterized matrix recovery,
  2018.

\bibitem[Maddison et~al.(2017)Maddison, Mnih, and Teh]{Maddison17}
Maddison, C.~J., Mnih, A., and Teh, Y.~W.
\newblock The concrete distribution: a continuous relaxation of discrete random
  variables.
\newblock \emph{International Conference on Learning Representations}, 2017.

\bibitem[Majeed \& Hutter(2018)Majeed and Hutter]{Majeed18}
Majeed, S.~J. and Hutter, M.
\newblock Performance guarantees for homomorphisms beyond markov decision
  processes.
\newblock \emph{CoRR}, abs/1811.03895, 2018.

\bibitem[Mao et~al.(2016)Mao, Alizadeh, Menache, and Kandula]{Mao2016}
Mao, H., Alizadeh, M., Menache, I., and Kandula, S.
\newblock Resource management with deep reinforcement learning.
\newblock 2016.

\bibitem[McAllester(1998)]{McAllester1998}
McAllester, D.~A.
\newblock Some pac-bayesian theorems.
\newblock \emph{Conference on Learning Theory (COLT)}, 1998.

\bibitem[McAllester(2003)]{Mcallester03}
McAllester, D.~A.
\newblock Simplified pac-bayesian margin bounds.
\newblock \emph{Conference on Learning Theory (COLT)}, 2003.

\bibitem[Mirhoseini et~al.(2018)Mirhoseini, Goldie, Pham, Steiner, Le, and
  Dean]{Mirhoseini18}
Mirhoseini, A., Goldie, A., Pham, H., Steiner, B., Le, Q.~V., and Dean, J.
\newblock Hierarchical planning for device placement.
\newblock 2018.
\newblock URL \url{https://openreview.net/pdf?id=Hkc-TeZ0W}.

\bibitem[Mnih et~al.(2015)Mnih, Kavukcuoglu, Silver, Rusu, Veness, Bellemare,
  Graves, Riedmiller, Fidjeland, Ostrovski, Petersen, Beattie, Sadik,
  Antonoglou, King, Kumaran, Wierstra, Legg, and Hassabis]{Mnih15}
Mnih, V., Kavukcuoglu, K., Silver, D., Rusu, A.~A., Veness, J., Bellemare,
  M.~G., Graves, A., Riedmiller, M., Fidjeland, A.~K., Ostrovski, G., Petersen,
  S., Beattie, C., Sadik, A., Antonoglou, I., King, H., Kumaran, D., Wierstra,
  D., Legg, S., and Hassabis, D.
\newblock Human-level control through deep reinforcement learning.
\newblock \emph{Nature}, 2015.

\bibitem[Mnih et~al.(2016)Mnih, Badia, Mirza, Graves, Lillicrap, Harley,
  Silver, and Kavukcuoglu]{mnih16}
Mnih, V., Badia, A.~P., Mirza, M., Graves, A., Lillicrap, T., Harley, T.,
  Silver, D., and Kavukcuoglu, K.
\newblock Asynchronous methods for deep reinforcement learning.
\newblock \emph{International Conference on Machine Learning}, 2016.

\bibitem[Mohri \& Medina(2012)Mohri and Medina]{Mohri12}
Mohri, M. and Medina, A.~M.
\newblock New analysis and algorithm for learning with drifting distributions.
\newblock \emph{Algorithmic Learning Theory}, 2012.

\bibitem[Nair et~al.(2015)Nair, Srinivasan, Blackwell, Alcicek, Fearon, Maria,
  Panneershelvam, Suleyman, Beattie, Petersen, Legg, Mnih, Kavukcuoglu, and
  Silver]{Nair15}
Nair, A., Srinivasan, P., Blackwell, S., Alcicek, C., Fearon, R., Maria, A.~D.,
  Panneershelvam, V., Suleyman, M., Beattie, C., Petersen, S., Legg, S., Mnih,
  V., Kavukcuoglu, K., and Silver, D.
\newblock Massively parallel methods for deep reinforcement learning.
\newblock \emph{CoRR}, abs/1507.04296, 2015.
\newblock URL \url{http://arxiv.org/abs/1507.04296}.

\bibitem[Neyshabur et~al.(2018)Neyshabur, Bhojanapalli, and
  Srebro]{Neyshabur2017a}
Neyshabur, B., Bhojanapalli, S., and Srebro, N.
\newblock A pac-bayesian approach to spectrally-normalized margin bounds for
  neural networks.
\newblock \emph{International Conference on Learning Representations (ICLR)},
  2018.

\bibitem[OpenAI(2018)]{OpenAI_dota}
OpenAI.
\newblock Openai five.
\newblock \url{https://blog.openai.com/openai-five/}, 2018.

\bibitem[Packer et~al.(2018)Packer, Gao, Kos, Kr\"{a}henb\"{u}hl, Koltun, and
  Song]{Packer18}
Packer, C., Gao, K., Kos, J., Kr\"{a}henb\"{u}hl, P., Koltun, V., and Song, D.
\newblock Assessing generalization in deep reinforcement learning.
\newblock \emph{CoRR}, 2018.
\newblock URL \url{https://arxiv.org/abs/1810.12282}.

\bibitem[R{\"u}ckstie{\ss} et~al.(2010)R{\"u}ckstie{\ss}, Sehnke, Schaul,
  Wierstra, Sun, and Schmidhuber]{Ruckstie2010}
R{\"u}ckstie{\ss}, T., Sehnke, F., Schaul, T., Wierstra, D., Sun, Y., and
  Schmidhuber, J.
\newblock Exploring parameter space in reinforcement learning.
\newblock \emph{Paladyn}, 2010.

\bibitem[Shalev-Shwartz \& Ben-David(2014)Shalev-Shwartz and
  Ben-David]{Shalev-Shwartz2014}
Shalev-Shwartz, S. and Ben-David, S.
\newblock \emph{Understanding Machine Learning: From Theory to Algorithms}.
\newblock Cambridge University Press, New York, NY, USA, 2014.
\newblock ISBN 1107057132, 9781107057135.

\bibitem[Shani et~al.(2005)Shani, Brafman, and Heckerman]{Shani05}
Shani, G., Brafman, R.~I., and Heckerman, D.
\newblock An mdp-based recommender system.
\newblock \emph{The Journal of Machine Learning Research}, 2005.

\bibitem[Silver et~al.(2016)Silver, Huang, Maddison, Guez, Sifre, van~den
  Driessche, Schrittwieser, Antonoglou, Panneershelvam, Lanctot, Dieleman,
  Grewe, Nham, Kalchbrenner, Sutskever, Lillicrap, Leach, Kavukcuoglu, Graepel,
  and Hassabis]{Silver16}
Silver, D., Huang, A., Maddison, C.~J., Guez, A., Sifre, L., van~den Driessche,
  G., Schrittwieser, J., Antonoglou, I., Panneershelvam, V., Lanctot, M.,
  Dieleman, S., Grewe, D., Nham, J., Kalchbrenner, N., Sutskever, I.,
  Lillicrap, T., Leach, M., Kavukcuoglu, K., Graepel, T., and Hassabis, D.
\newblock Mastering the game of go with deep neural networks and tree search.
\newblock \emph{Nature}, 2016.

\bibitem[Silver et~al.(2017)Silver, Hubert, Schrittwieser, Antonoglou, Lai,
  Guez, Lanctot, Sifre, Kumaran, Graepel, Lillicrap, Simonyan, and
  Hassabis]{Silver17}
Silver, D., Hubert, T., Schrittwieser, J., Antonoglou, I., Lai, M., Guez, A.,
  Lanctot, M., Sifre, L., Kumaran, D., Graepel, T., Lillicrap, T.~P., Simonyan,
  K., and Hassabis, D.
\newblock Mastering chess and shogi by self-play with a general reinforcement
  learning algorithm.
\newblock \emph{CoRR}, 2017.
\newblock URL \url{http://arxiv.org/abs/1712.01815}.

\bibitem[Strehl et~al.(2009)Strehl, Li, and Littman]{Strehl09}
Strehl, A.~L., Li, L., and Littman, M.~L.
\newblock Reinforcement learning in finite mdps: Pac analysis.
\newblock \emph{Journal of Machine Learning Research}, 2009.

\bibitem[Sutton \& Barto.(1998)Sutton and Barto.]{Sutton98}
Sutton, R. and Barto., A.
\newblock \emph{Reinforcement Learning: An Introduction.}
\newblock MIT Press, 1998.

\bibitem[Sutton(1995)]{Sutton95}
Sutton, R.~S.
\newblock Generalization in reinforcement learning: Successful examples using
  sparse coarse coding.
\newblock 1995.

\bibitem[Taylor \& Stone(2009)Taylor and Stone]{Taylor09}
Taylor, M.~E. and Stone, P.
\newblock Transfer learning for reinforcement learning domains: A survey.
\newblock \emph{J. Mach. Learn. Res.}, 2009.

\bibitem[Tokui \& Sato(2016)Tokui and Sato]{Tokui16}
Tokui, S. and Sato, I.
\newblock Reparameterization trick for discrete variables.
\newblock \emph{CoRR}, 2016.
\newblock URL \url{https://arxiv.org/abs/1611.01239}.

\bibitem[van~der Vaart.(1998)]{Vaart98}
van~der Vaart., A.
\newblock \emph{Asymptotic Statistics..}
\newblock Cambridge, 1998.

\bibitem[Vinyals et~al.(2017)Vinyals, Ewalds, Bartunov, Georgiev, Vezhnevets,
  Yeo, Makhzani, K{\"{u}}ttler, Agapiou, Schrittwieser, Quan, Gaffney,
  Petersen, Simonyan, Schaul, van Hasselt, Silver, Lillicrap, Calderone, Keet,
  Brunasso, Lawrence, Ekermo, Repp, and Tsing]{Vinyals17}
Vinyals, O., Ewalds, T., Bartunov, S., Georgiev, P., Vezhnevets, A.~S., Yeo,
  M., Makhzani, A., K{\"{u}}ttler, H., Agapiou, J., Schrittwieser, J., Quan,
  J., Gaffney, S., Petersen, S., Simonyan, K., Schaul, T., van Hasselt, H.,
  Silver, D., Lillicrap, T.~P., Calderone, K., Keet, P., Brunasso, A.,
  Lawrence, D., Ekermo, A., Repp, J., and Tsing, R.
\newblock Starcraft {II:} {A} new challenge for reinforcement learning.
\newblock \emph{CoRR}, 2017.
\newblock URL \url{http://arxiv.org/abs/1708.04782}.

\bibitem[Wang et~al.(2018{\natexlab{a}})Wang, Keskar, Xiong, and
  Socher]{Wang18}
Wang, H., Keskar, N.~S., Xiong, C., and Socher, R.
\newblock Identifying generalization properties in neural networks.
\newblock 2018{\natexlab{a}}.
\newblock URL \url{https://openreview.net/forum?id=BJxOHs0cKm}.

\bibitem[Wang et~al.(2018{\natexlab{b}})Wang, Liu, and Li]{WangJ18}
Wang, J., Liu, Y., and Li, B.
\newblock Reinforcement learning with perturbed rewards.
\newblock \emph{CoRR}, abs/1810.01032, 2018{\natexlab{b}}.
\newblock URL \url{http://arxiv.org/abs/1810.01032}.

\bibitem[Whiteson et~al.(2011)Whiteson, Tanner, Taylor, and Stone]{Whiteson11}
Whiteson, S., Tanner, B., Taylor, M.~E., and Stone, P.
\newblock Protecting against evaluation overfitting in empirical reinforcement
  learning.
\newblock \emph{2011 IEEE Symposium on Adaptive Dynamic Programming and
  Reinforcement Learning (ADPRL)}, 2011.

\bibitem[Williams(1992)]{Williams92}
Williams, R.~J.
\newblock Simple statistical gradient-following algorithms for connectionist
  reinforcement learning.
\newblock \emph{Machine Learning}, 1992.

\bibitem[Zhan \& Taylor(2015)Zhan and Taylor]{Zhan15}
Zhan, Y. and Taylor, M.~E.
\newblock Online transfer learning in reinforcement learning domains.
\newblock \emph{CoRR}, abs/1507.00436, 2015.

\bibitem[Zhang et~al.(2018{\natexlab{a}})Zhang, Ballas, and Pineau]{Zhang18b}
Zhang, A., Ballas, N., and Pineau, J.
\newblock A dissection of overfitting and generalization in continuous
  reinforcement learning.
\newblock \emph{CoRR}, 2018{\natexlab{a}}.
\newblock URL \url{https://arxiv.org/abs/1806.07937}.

\bibitem[Zhang et~al.(2018{\natexlab{b}})Zhang, Vinyals, Munos, and
  Bengio]{Zhang18}
Zhang, C., Vinyals, O., Munos, R., and Bengio, S.
\newblock A study on overfitting in deep reinforcement learning.
\newblock \emph{CoRR}, 2018{\natexlab{b}}.
\newblock URL \url{http://arxiv.org/abs/1804.06893}.

\end{thebibliography}
\bibliographystyle{icml2019}

\clearpage

\appendix
\section{Proof of Lemma 3}
\begin{lemma*}
For Reparameterizable RL, given assumptions \ref{assump:lipschitzT}, \ref{assump:lipschitzP}, and \ref{assump:lipschitzR}, the empirical reward $R$ defined in (\ref{eqn:emp-reward}), as a function of the parameter $\theta$, has a Lipschitz constant of 
\begin{align}
\beta = \sum_{t=0}^T \gamma^t L_rL_{t_2}L_{\pi2}\frac{\nu^t-1}{\nu-1}\nonumber
\end{align}
where $\nu=L_{t1}+L_{t2}L_{\pi1}$.
\end{lemma*}

\begin{proof}
Let's denote $s_t'=s_t(\theta')$, and $s_t=s_t(\theta)$. We start by investigating the policy function across different time steps:
\begin{align}
    &\|\pi(s_t';\theta')-\pi(s_t;\theta)\| \nonumber\\
    &= \|\pi(s_t';\theta')-\pi(s_t;\theta')+\pi(s_t;\theta')-\pi(s_t;\theta)\|\nonumber\\
    &\leq \|\pi(s_t';\theta')-\pi(s_t;\theta')\|+\|\pi(s_t;\theta')-\pi(s_t;\theta)\|\nonumber\\
    &\leq L_{\pi1}\|s_t'- s_t\|+L_{\pi2}\|\theta'-\theta\|\label{app:pi}
\end{align}
The first inequality is the triangle inequality, and the second is from our Lipschitz assumption \ref{assump:lipschitzP}.

If we look at the change of states as the episode proceeds:
\begin{align}
    &\|s_t'-s_t\|\nonumber\\
    &=\|\mathcal{T}(s_{t-1}', \pi(s_{t-1}';\theta'), \xi_{t-1})-\mathcal{T}(s_{t-1}, \pi(s_{t-1};\theta), \xi_{t-1})\|\nonumber\\
    &\leq\|\mathcal{T}(s_{t-1}', \pi(s_{t-1}';\theta'), \xi_{t-1}) - \mathcal{T}(s_{t-1}, \pi(s_{t-1}';\theta'), \xi_{t-1})\| \nonumber\\&+ \|\mathcal{T}(s_{t-1}, \pi(s_{t-1}';\theta'), \xi_{t-1})-\mathcal{T}(s_{t-1}, \pi(s_{t-1};\theta), \xi_{t-1})\|\nonumber\\
    &\leq L_{t1}\|s_{t-1}'-s_{t-1}\| + L_{t2}\|\pi(s_{t-1}';\theta')- \pi(s_{t-1};\theta)\|\label{app:s}
\end{align}

Now combine both (\ref{app:pi}) and (\ref{app:s}),
\begin{align}
    &\|s_t'-s_t\|\nonumber\\
    &\leq L_{t1}\|s_{t-1}'-s_{t-1}\|\nonumber\\ &+ L_{t2}(L_{\pi1}\|s_{t-1}'- s_{t-1}\|+L_{\pi2}\|\theta'-\theta\|) \nonumber\\
    &\leq (L_{t1}+ L_{t2}L_{\pi1})\|s_{t-1}'-s_{t-1}\|+L_{t2}L_{\pi2}\|\theta'-\theta\| \nonumber
\end{align}

In the initialization, we know $s_0'=s_0$ since the initialization process does not involve any computation using the parameter $\theta$ in the policy $\pi$.

By recursion, we get
\begin{align}
\|s_t'-s_t\|&\leq L_{t_2}L_{\pi2}\|\theta'-\theta\|\sum_{t=0}^{t-1}(L_{t1}+L_{t2}L_{\pi1})^t\nonumber\\
&=L_{t_2}L_{\pi2}\frac{\nu^t-1}{\nu-1}\|\theta'-\theta\|\nonumber
\end{align}
where $\nu=L_{t1}+L_{t2}L_{\pi1}$.

By assumption \ref{assump:lipschitzR}, $r(s)$ is $L_r$-Lipschitz, so
\begin{align}
    \|r(s_t')-r(s_t)\|&\leq L_r\|s_t'-s_t\|\nonumber\\
    &\leq L_rL_{t_2}L_{\pi2}\frac{\nu^t-1}{\nu-1}\|\theta'-\theta\|\nonumber
\end{align}

So the reward 
\begin{align}
    &|R(s')-R(s)| = |\sum_{t=0}^T \gamma^t r(s_t')-\sum_{t=0}^T \gamma^t r(s_t)|\nonumber\\
    &\leq |\sum_{t=0}^T \gamma^t (r(s_t')- r(s_t))|
    \leq \sum_{t=0}^T \gamma^t |r(s_t')- r(s_t))|\nonumber\\   
    &\leq \sum_{t=0}^T \gamma^t L_rL_{t_2}L_{\pi2}\frac{\nu^t-1}{\nu-1}\|\theta'-\theta\|=\beta\|\theta'-\theta\|\nonumber
\end{align}

\end{proof}

\section{Proof of Lemma \ref{lemma:initchange}}

\begin{lemma*}
In reparameterizable RL, suppose the initialization function $\mathcal{I}'$ in the test environment satisfies $\|(\mathcal{I}'-\mathcal{I})(\xi)\|\leq \delta$, and the transition function is the same for both training and testing environment. If assumptions (\ref{assump:lipschitzT}), (\ref{assump:lipschitzP}), and (\ref{assump:lipschitzR}) hold then 
\begin{align}
    |\mathbb{E}_\xi[R(s(\xi;\mathcal{I}'))] - \mathbb{E}_\xi[R(s(\xi;\mathcal{I}))]|\leq\nonumber\\ \sum_{t=0}^{T} \gamma^{t} L_r (L_{t1}+L_{t2}L_{\pi1})^t\delta \nonumber
\end{align}
\end{lemma*}
\begin{proof}
Denote the states at time $t$ with $\mathcal{I}'$ as the initialization function as $s_t'$.
Again we look at the difference between $s_t'$ and $s_t$. By triangle inequality and assumptions \ref{assump:lipschitzT} and \ref{assump:lipschitzP},
\begin{align}
    &\|s_t'-s_t\|\nonumber\\
    &=\|\mathcal{T}(s_{t-1}', \pi(s_{t-1}'), \xi_{t-1}) - \mathcal{T}(s_{t-1}, \pi(s_{t-1}), \xi_{t-1})\|\nonumber\\
    &\leq\|\mathcal{T}(s_{t-1}', \pi(s_{t-1}'), \xi_{t-1}) -\mathcal{T}(s_{t-1}, \pi(s_{t-1}'), \xi_{t-1})\|\nonumber\\ &+\|\mathcal{T}(s_{t-1}, \pi(s_{t-1}'), \xi_{t-1})-\mathcal{T}(s_{t-1}, \pi(s_{t-1}), \xi_{t-1})\|\nonumber\\
    &\leq L_{t1}\|s_{t-1}'-s_{t-1}\| +L_{t2}\|\pi(s_{t-1}')- \pi(s_{t-1})\|\nonumber\\
    &\leq L_{t1}\|s_{t-1}'-s_{t-1}\| +L_{t2}L_{\pi1}\|s_{t-1}'- s_{t-1}\|\nonumber\\
    &= (L_{t1} +L_{t2}L_{\pi1})\|s_{t-1}'- s_{t-1}\|\nonumber\\
    &\leq (L_{t1} +L_{t2}L_{\pi1})^t\|s_0'-s_0\|\nonumber\\
    &\leq (L_{t1} +L_{t2}L_{\pi1})^t\delta\nonumber
\end{align}
where the last inequality is due to the assumption that 
\begin{align}
    \|s_0'-s_0\| = \|\mathcal{I}'(\xi)-\mathcal{I}(\xi)\|\leq \delta\nonumber
\end{align}

Also since $r(s)$ is also Lipschitz,
\begin{align}
    &|R(s')-R(s)|= |\sum_{t=0}^T \gamma^t r(s_t')-\sum_{t=0}^T \gamma^t r(s_t)|\nonumber\\
    &\leq \sum_{t=0}^T \gamma^t |r(s_t')- r(s_t)|\leq\sum_{t=0}^T \gamma^t L_r\|s_t'- s_t\|\nonumber\\ 
    &\leq L_r\delta\sum_{t=0}^T\gamma^t(L_{t1} +L_{t2}L_{\pi1})^t\nonumber
\end{align}
The argument above holds for any given random input $\xi$, so
\begin{align}
    &|\mathbb{E}_\xi[R(s'(\xi)]-\mathbb{E}_\xi[R(s(\xi)]|\nonumber\\
    &\leq \left|\int_\xi \left(R(s'(\xi))-R(s(\xi))\right)\right|\nonumber\\
    &\leq \int_\xi \left|R(s'(\xi))-R(s(\xi))\right|\nonumber\\    
    &\leq L_r\delta\sum_{t=0}^T\gamma^t(L_{t1} +L_{t2}L_{\pi1})^t\nonumber
\end{align}

\end{proof}

\section{Proof of Lemma \ref{lemma:external}}
\begin{lemma*}
In reparameterizable RL, suppose the transition $\mathcal{T}'$ in the test environment satisfies $\forall x, y,z, \|(\mathcal{T}'-\mathcal{T})(x,y,z)\|\leq \delta$, and the initialization is the same for both the training and testing environment. If assumptions (\ref{assump:lipschitzT}), (\ref{assump:lipschitzP}) and (\ref{assump:lipschitzR}) hold then 
\begin{align}
    |\mathbb{E}_\xi[R(s(\xi;\mathcal{T}'))] - \mathbb{E}_\xi[R(s(\xi;\mathcal{T}))]|\leq \sum_{t=0}^T \gamma^{t} L_r \frac{1-\nu^t}{1-\nu}\delta 
\end{align}
where $\nu=L_{t1}+L_{t2}L_{\pi1}$
\end{lemma*}
\begin{proof}
Again let's denote the state at time t with the new transition function $\mathcal{T}'$ as $s_t'$, and the state at time t with the original transition function $\mathcal{T}$ as $s_t$, then
\begin{align}
&\|s_t'-s_t\|\nonumber\\
&=\|\mathcal{T}'(s_{t-1}',\pi(s_{t-1}'), \xi_{t-1}) - \mathcal{T}(s_{t-1},\pi(s_{t-1}),\xi_{t-1}) \|\nonumber\\
&\leq\|\mathcal{T}'(s_{t-1}',\pi(s_{t-1}'), \xi_{t-1}) -\mathcal{T}'(s_{t-1},\pi(s_{t-1}), \xi_{t-1})\|+\nonumber\\ &\|\mathcal{T}'(s_{t-1},\pi(s_{t-1}), \xi_{t-1})-\mathcal{T}(s_{t-1},\pi(s_{t-1}),\xi_{t-1}) \|\nonumber\\
&\leq \|\mathcal{T}'(s_{t-1}',\pi(s_{t-1}'), \xi_{t-1})-\mathcal{T}'(s_{t-1},\pi(s_{t-1}'), \xi_{t-1})\|\nonumber\\&+\| \mathcal{T}'(s_{t-1},\pi(s_{t-1}'), \xi_{t-1})-\mathcal{T}'(s_{t-1},\pi(s_{t-1}), \xi_{t-1})\|+ \delta\nonumber\\
&\leq L_{t1}\|s_{t-1}'-s_{t-1}\| + L_{t2}\|\pi(s_{t-1}')-\pi(s_{t-1})\| + \delta\nonumber\\
&\leq L_{t1}\|s_{t-1}'-s_{t-1}\| + L_{t2}L_{\pi1}\|s_{t-1}'-s_{t-1}\| + \delta\nonumber\\
&= (L_{t1} + L_{t2}L_{\pi1})\|s_{t-1}'-s_{t-1}\| + \delta\nonumber
\end{align}
Again we have the initialization condition
\begin{align}
    s_0'=s_0\nonumber
\end{align}
since the initialization procedure $\mathcal{I}$ stays the same.
By recursion we have
\begin{align}
    \|s_t'-s_t\|\leq \delta \sum_{t=0}^{t-1}(L_{t1}+L_{t2}L_{\pi1})^t
\end{align}
By assumption \ref{assump:lipschitzR},
\begin{align}
    &|R(s')-R(s)|= |\sum_{t=0}^T \gamma^t r(s_t')-\sum_{t=0}^T \gamma^t r(s_t)|\nonumber\\
    &\leq \sum_{t=0}^T \gamma^t |r(s_t')- r(s_t)|\leq\sum_{t=0}^T \gamma^t L_r\|s_t'- s_t\|\nonumber\\ 
    &\leq L_r\delta\sum_{t=0}^T\gamma^t\left(\sum_{k=0}^{t-1}(L_{t1} +L_{t2}L_{\pi1})^k\right)\nonumber\\
    &\leq L_r\delta\sum_{t=0}^T\gamma^t\frac{\nu^t-1}{\nu-1}\nonumber    
\end{align}
where $\nu=L_{t1} +L_{t2}L_{\pi1}$.
Again the argument holds for any given random input $\xi$, so 
\begin{align}
    &|\mathbb{E}_\xi[R(s'(\xi)]-\mathbb{E}_\xi[R(s(\xi)]|\nonumber\\
    &\leq \left|\int_\xi \left(R(s'(\xi))-R(s(\xi))\right)\right|\nonumber\\
    &\leq \int_\xi \left|R(s'(\xi))-R(s(\xi))\right|\nonumber\\
    &\leq L_r\delta\sum_{t=0}^T\gamma^t\frac{\nu^t-1}{\nu-1}\nonumber
\end{align}
\end{proof}

\section{Proof of Theorem \ref{thm: main}}
\begin{theorem*}
In \textbf{reparameterizable RL}, suppose the transition $\mathcal{T}'$ in the test environment satisfies $\forall x, y,z, \|(\mathcal{T}'-\mathcal{T})(x,y,z)\|\leq \zeta$, and suppose the initialization function $\mathcal{I}'$ in the test environment satisfies $\forall \xi, \|(\mathcal{I}'-\mathcal{I})(\xi)\|\leq \epsilon$. If assumptions (\ref{assump:lipschitzT}), (\ref{assump:lipschitzP}) and (\ref{assump:lipschitzR}) hold, the peripheral random variables $\xi^i$ for each episode are i.i.d., and the reward is bounded $|R(\boldsymbol s)|\leq c/2$, then with probability at least $1-\delta$, for all policy $\pi\in\Pi$,
\begin{align}
    &|\mathbb{E}_\xi[R(\boldsymbol s(\xi;\pi, \mathcal{T}', \mathcal{I}'))] - \frac{1}{n}\sum_i R(\boldsymbol s(\xi^i;\pi,\mathcal{T}, \mathcal{I}))|\nonumber\\
    &\leq  Rad(R_{\pi,\mathcal{T},\mathcal{I}}) + L_r \zeta\sum_{t=0}^T \gamma^{t} \frac{\nu^t-1}{\nu-1} + L_r \epsilon\sum_{t=0}^{T} \gamma^{t}  \nu^t\nonumber\\&+ O\left(c\sqrt{\frac{\log (1/\delta)}{n}}\right)\nonumber
\end{align}
where $\nu=L_{t1}+L_{t2}L_{\pi1}$, and  $$Rad(R_{\pi,\mathcal{T},\mathcal{I}}) = \mathbb{E}_{\xi}\mathbb{E}_\sigma\left[\sup_\pi \frac{1}{n}\sum_{i=1}^n \sigma_i R({\boldsymbol s}^i(\xi^i;\pi,\mathcal{T}, \mathcal{I}))\right]$$ is the Rademacher complexity of $R(\boldsymbol s(\xi;\pi,\mathcal{T},\mathcal{I}))$ under the training transition $\mathcal{T}$, the training initialization $\mathcal{I}$, and $n$ is the number if training episodes.
\end{theorem*}
\begin{proof}
Note
\begin{align}
    &|\frac{1}{n}\sum_i R(s(\xi^i;\pi,\mathcal{T}, \mathcal{I}))-\mathbb{E}_\xi[R(s(\xi;\pi, \mathcal{T}', \mathcal{I'}))]|\nonumber\\
    &\leq |\frac{1}{n}\sum_i R(s(\xi^i;\pi,\mathcal{T}, \mathcal{I}))-\mathbb{E}_\xi[R(s(\xi;\pi, \mathcal{T}, \mathcal{I}))]|\nonumber\\
    &+|\mathbb{E}_\xi[R(s(\xi;\pi, \mathcal{T}, \mathcal{I}))]-\mathbb{E}_\xi[R(s(\xi;\pi, \mathcal{T}', \mathcal{I}))]|\nonumber\\
    & + |\mathbb{E}_\xi[R(s(\xi;\pi, \mathcal{T}', \mathcal{I}))]-\mathbb{E}_\xi[R(s(\xi;\pi, \mathcal{T}', \mathcal{I}'))]|\nonumber
\end{align}
Then theorem \ref{thm: main} is a direct consequence of Lemma \ref{lemma:rad}, Lemma \ref{lemma:initchange}, and Lemma \ref{lemma:external}.
\end{proof}
\end{document}